\documentclass[pmlr]{jmlr}

\usepackage{longtable}
\usepackage{algpseudocode}
\usepackage{titlesec}
\usepackage{float}
\SetKwInput{KwInput}{Input}
\SetKwInput{KwOutput}{Output}  
\newcommand\myeq{\mathrel{\stackrel{\makebox[0pt]{\mbox{\normalfont\tiny def}}}{=}}}

\usepackage{booktabs}
\usepackage[load-configurations=version-1]{siunitx} 

\theorembodyfont{\upshape}
\theoremheaderfont{\scshape}
\theorempostheader{:}
\theoremsep{\newline}

\jmlrvolume{1}
\jmlryear{2022}
\jmlrworkshop{Learning and Automata Workshop}

\title[Marginal Inference queries in HMMs under CFG constraints]{Marginal Inference queries in Hidden Markov Models under context-free grammar constraints }

  \author{\Name{Reda Marzouk} \\ \Email{mohamed-reda.marzouk@univ-nantes.fr}\\ 
   \Name{Colin de La Higuera} \\ \Email{cdlh@univ-nantes.fr}\\
   \addr Université de Nantes/Laboratoire des Sciences du Numérique de Nantes (LS2N, UMR CNRS 6004), Nantes, France}

\begin{document}

\maketitle
\begin{abstract}
    The primary use of any probabilistic model involving a set of random variables is to run inference and sampling queries on it. Inference queries in classical probabilistic models is concerned by the computation of marginal or conditional probabilities of events given as an input. When the probabilistic model is sequential, more sophisticated marginal inference queries involving complex grammars may be of interest in fields such as computational linguistics and NLP. In this work, we address the question of computing the likelihood of context-free grammars (CFGs) in Hidden Markov Models (HMMs). We provide a dynamic algorithm for the exact computation of the likelihood for the class of unambiguous context-free grammars. We show that the problem is NP-Hard, even with the promise that the input CFG has a degree of ambiguity less than or equal to 2. We then propose a fully polynomial randomized approximation scheme (FPRAS) algorithm to approximate the likelihood for the case of polynomially-bounded ambiguous CFGs.
\end{abstract}
Due to their simplicity and ease  of interpretation, Hidden Markov Models (HMMs) represent an important class of dynamic probabilistic models finding  its applications in many fields (bioinformatics, NLP, speech processing etc.). With regard to making inference in HMMs, two popular algorithms were proposed in the litterature: The forward algorithm and the Viterbi algorithm. The former computes the likelihood of generating a given sequence, while the latter is dedicated to the decoding problem. Algorithms for more elaborate marginal queries can be obtained by observing that HMMs restricted to a bounded support are nothing but tree-structured bayesian networks. Luckily, exact inference in this class of probabilistic models is known to be tractable (\cite{koller09}). With this correspondance, one can readily obtain procedures to compute infix (or, more generally island) probabilities in HMMs by straightforward reduction. In most applications where HMMs are employed, this suite of marginal inference queries proved to be sufficient for the task of interest. Yet, in fields such as computationnal linguistics, deeper structures of sequential data involving complex grammars are of interest. For instance, when HMMs are used in NLP tasks such as language modeling, we might be interested in extracting emergent grammars in the language model. Other concerns are related to the dicriminative bias in language models. Indeed, recent studies showed that trained language models are subject of capturing negative stereotypes present in the corpus used for training (\cite{sheng19}), a problem that raises a major ethical concern regarding the deployment of these models in real-case situations. The goal of this article is to explore the perspective of widening the scope of marginal inference queries we can make on HMMs by bringing to the picture languages generated by sophisticated grammars. Informally, let $\mathcal{R}$ be a class of machines/grammars inducing languages over the alphabet formed by the observation space of the HMM, the $\mathcal{R}$-likelihood problem can be stated as follows: Given an HMM $M$, a machine/grammar $R \in \mathcal{R}$ and an integer $L>0$: How to compute the likelihood of generating a sequence of length $L$ from $M$ accepted by $R$? \\
From the perspective of formal language theory, this problem can be seen as a generalization of the classical problem of counting the number of words accepted by a given machine/grammar. This counting problem is easily reducible to our problem, by choosing $M$ to be the trivial single-state HMM that assigns equal mass probability to all words at a given length. Therefore, by this straightforward reduction, one can show that the $\mathcal{R}$-likelihood problem is $\#$P-Hard when $\mathcal{R}$ is the class of nondeterministic finite automata (\cite{kannan95}). A fortiori, same negative results hold for general context-free grammars. \\
  Positive results of the $\mathcal{R}$-likelihood problem can be found in the litterature of computational linguistics interested in the closely related problem of computing the weight of different classes of languages in stochastic grammars such as probabilistic context-free grammars and probabilistic finite automata (PFAs). Indeed, HMMs bounded to a given length can be efficiently converted to a PFA. Therefore, an algorithm designed for PFAs to compute the weight of languages represented by a class of machines/grammars $\mathcal{R}$ can be adapted to our problem. Earlier works adressed the problem of computing the weight of prefixes, suffixes and infixes in PFAs (\cite{corazza81}, \cite{ana00}). Recently, \cite{cognetta18, cognetta19} proposed an incremental algorithm for computing the weight of regular languages in PFAs, where the regular languages are represented by \textit{unambiguous} regular expressions. \\
  In this article, we shall be interested in computing the likelihood of languages generated by different subclasses of  CFGs from HMMs (i.e. when $\mathcal{R}$ is a subclass of CFGs). The studied subsclasses will correspond to the hirerarchy of CFGs according to their ambiguity degree. The outline of the article is as follows: In Section 1, we present algebric descriptions of machines/grammars of interest in this article, that is CFGs and HMMs. In Section 2, we present a polynomial-time dynamic algorithm for solving the likelihood problem for the case of unambiguous context-free grammars (UCFG). In Section 3, we prove that the likelihood problem is NP-Hard even with the promise that the input CFG has an ambiguity degree less or equal to 2. We give afterwards a randomized procedure that provides FPRAS guarantees of the likelihood problem for the case of CFGs with polynomial ambiguity degree. 
\section{Algebric descriptions of CFGs/HMMs}
Let $\Sigma$ be a finite alphabet, and $\epsilon$ a special symbol representing the empty word. $\Sigma^{+}$ is the set of all strings formed by alphabet $\Sigma$, and $\Sigma^{*} = \Sigma^{+} \cup \{\epsilon\}$. For a given string $w \in \Sigma^{*}$, we denote by $|w|$ its length, by $w_{i:j}$ the substring of $w$ that spans from the $i$-th letter to the $j$-th letter in $w$, and $w_{i}$ to refer to its $i$-th letter.  For a given integer $n > 0$, we denote by $[n]$ the set of all integers from $1$ to $n$. A language $f$ is a mapping from $\Sigma^{*}$ to $\mathbb{R}$. When the image of the mapping is binary, then we call the language binary, in which case the language represents a subset of $\Sigma^{*}$. We extend the definition of languages to binary languages over $\Sigma^{*}$: For a subset $L \subseteq \Sigma^{*}$, $f(L) \myeq \sum\limits_{w \in L} f(w)$ which may be infinite. The indicator function of a set $X$ is denoted by $\delta_{X}$. When a language $f$ is described by a machine/grammar $M$, we'll use the notation $f_{M}$ to refer to the language induced by $M$. \\
$\bullet$ \textbf{Linear Algebra.} For a subset $I \subseteq [n]$, the vector of dimension $n$ in which all elements indexed by an element in $I$ are equal to 1 and 0 otherwise will be denoted $e_{I}^{(n)}$. When the dimension of the vector is clear from context, we omit the superscript. The transpose of a vector $v$ will be denoted by $v^{T}$. For two matrices $A, ~B$, the kronecker product will be denoted as $A \otimes B$ (\cite{zhang13}).
For a matrix $\mathbb{R}^{n \times m}$, $vec(A)$ denotes the vector in $\mathbb{R}^{n.m}$ that results from stacking columns of $A$ vertically. A generalization of vectors and matrices is tensors. A tensor is a multidimensional array $\mathbb{R}^{n_{1} \times..\times n_{d}}$, where $d$ is called the order of a tensor. The $(i_{1},..i_{d})$-th element of a tensor $\mathcal{T}$ will be denoted $\mathcal{T}[i_{1},..i_{d}]$. The mode-$k$ matriciziation of a tensor $\mathcal{T}$, denoted $\mathcal{T}_{k}$, is a reshape of a tensor into a matrix in $\mathbb{R}^{n_{k} \times n_{1}\dots n_{i-1}n_{i+1} \dots n_{d}}$. The columns $\mathcal{T}_{k}$ are the vectors 
$$\{\mathcal{T}[i_{1},i_{2} \dots i_{k-1},:,i_{k+1},i_{d}]: (i_{1},\dots i_{k-1},i_{k+1},\dots,n_{d}) \in \mathbb{R}^{n_{1}}\times..\mathbb{R}^{n_{k-1}}\times \mathbb{R}^{n_{k+1}}\times..\times \mathbb{R}^{n_{d}} \}$$
stacked in the inverse lexicographical order. The kronecker product to two tensors $\mathcal{T},~\mathcal{M}$ of the same order, denoted $\mathcal{M} \otimes \mathcal{T}$, is given as $(\mathcal{T} \otimes \mathcal{M})_{k} = \mathcal{T}_{k} \otimes \mathcal{M}_{k}$, for some $k \in [d]$. One can verify that the resulting tensor is independent of the chosen dimension. 

\subsection{Context-Free Grammars.} 
  A context-free grammar (CFG) $G$ in Chomsky form is represented by the tuple $G = <[n], \Sigma, R, S>$ where $[n]$ is the set of non-terminal symbols, $R$ is a set of production rules each having one of the following form: $i \rightarrow jk$ or $i \rightarrow \sigma$ where $i,j,k$ are non-terminals and $\sigma$ is a terminal symbol and $S \in [n]$ is the initial non-terminal symbol. The size of a CFG, denoted $|G|$, is equal to the cardinality of the set of production rules. A valid derivation tree of $G$ is a binary tree whose internal nodes are labelled by non-terminal symbols of $G$, and satisfies the following \textit{syntactic} constraints: \textit{(a)} The root node of the tree is labelled by $S$, \textit{(b)} For any internal node $N$ labelled by a non-terminal symbol $i$, the labels of its children nodes must correspond to a production rule in $R$ from left to right.
The binary language accepted by $G$, denoted $L_{G}$, is the set of strings $w \in \Sigma^{*}$ for which there exists at least one derivation tree whose leaf nodes are labelled by $w$ (reading from left to right). \\
Using close connections between CFGs and tree automata, an alternative algebric representation of CFGs will be particularly useful in the context of our work: 
\begin{definition}Let $G= <[n], \Sigma, R, S>$ be a CFG in Chomsky Form. The algebraic representation of $G$ is given by the tuple: $G = <[n], \pi, \mathcal{M}, \{\beta_{\sigma}\}_{\sigma \in \Sigma}>$, where: 
\begin{itemize}
\setlength{\itemsep}{1pt}
  \setlength{\parskip}{0pt}
  \setlength{\parsep}{0pt}
    \item $\pi \in \mathbb{R}^{n}$ is the initial vector equal to $e_{\{S\}}$,
    \item $\mathcal{M} \in \mathbb{R}^{n \times n \times n}$ is a tensor of order 3 where $\mathcal{T}(i,j,k) = 1$ if $i \rightarrow jk$ is in $R$, 0 otherwise, 
    \item $\beta_{\sigma} \in \mathbb{R}^{n}$ is a vector associated to each symbol $\sigma$ in the alphabet where $\beta_{\sigma}[i] = 1$ if $i \rightarrow \sigma$ is in $R$, 0 otherwise
\end{itemize}
\end{definition}
 The algebric representation of a CFG $G$ associates to each tree structure $t$ a vector $\beta_{t}$ using the following recursive formula: \\
$\bullet$ The tree associated to a terminal symbol $\sigma \in \Sigma$, denoted $t(\sigma)$, is equal to $\beta_{\sigma}$, \\
$\bullet$ For two trees $t_{1}, t_{2}$, the embedding vector of the tree formed by adding a new root node and binding $t_{1}$ (resp. $t_{2}$) in its left (resp. right) is equal to $\mathcal{M}_{1} (\beta_{t_{1}} \otimes \beta_{t_{2}})$ \\
$~~$ The language generated by the algebric representation of CFGs is given as: $f_{G}(w) = \pi^{T}.\sum\limits_{t: yield(t) = w} \beta_{t}$, which corresponds to
the number of valid derivation trees yielding a string $w \in \Sigma^{*}$.
This expression is not useful in practice as it involves the enumeration of all trees that yield a string $w$. We shall favor another way to compute $f_{G}$ that gives rise a more practical expression. Define the map $\beta_{G}: \Sigma^{*} \rightarrow \mathbb{R}^{n}$ where $\beta_{G}(w) = \sum\limits_{t: yield(t) = w} \beta_{t}$, we have:
\begin{equation} \label{betagw}
      \beta_{G}(w) = \mathcal{M}_{1} \cdot [ \sum\limits_{\substack{w_{1}, w_{2} \in \Sigma^{+} \\ w_{1}w_{2} = w}} \beta_{G}(w_{1}) \otimes \beta_{G}(w_{2})] 
\end{equation}
For $i \in [n]$, $\beta_{G}(w)[i]$ represents the number of derivation trees yielding $w$ whose root node is $i$. The expression \eqref{betagw} expresses the fact that a derivation tree whose yield $w$ is necessarly obtained by a concatenation of two substrees yielding a factorization of $w$. The growth function for a CFG $G$ is a mapping $\phi_{G}: \mathbb{N} \rightarrow \mathbb{N}$ such that $\phi_{G}(n) = \max\limits_{w \in \Sigma^{n}} f_{G}(w)$. The hierarchisation of CFGs according to their ambiguity degree was extensively studied in the literature as an attempt to identify subsclasses of CFGs with interesting theoretical and computational properties ( \cite{shamir71}, \cite{crestin72}, \cite{wich01, wich05}). We distinguish:
\begin{itemize}
  \setlength{\itemsep}{1pt}
  \setlength{\parskip}{0pt}
  \setlength{\parsep}{0pt}
    \item CFGs with constant ambiguity: $\sup\limits_{n \in \mathbb{N}} \phi_{G}(n) \leq k$, for some $k > 0$. If $k=1$, then the CFG is said to be \textit{unambiguous} (UCFG).
    \item CFGs with polynomial ambiguity (POLYCFG): $\phi_{G}(n) \leq \Theta(n^{k})$ for some $k > 0$,
    \item CFGs with exponential ambiguity (EXPCFG): $\phi_{G}(n) \leq 2^{\Theta(n)}$
\end{itemize}
\subsection{Hidden Markov Models:} 
The observable operator representation of a HMM is given by the triplet $A = <[n], \pi, \{A_{\sigma}\}_{\sigma \in \Sigma}>$, where $\pi$ is a probability distribution over $[n]$, and $\{A_{\sigma}\}_{\sigma \in \Sigma}$ are positive matrices in $\mathbb{R}^{n \times n}$ such that $\sum\limits_{\sigma \in \Sigma} A_{\sigma}$ is a stochastic matrix \footnote{A non-negative matrix $A \in \mathbb{R}^{n \times n}$ is said to be \textit{stochastic} if $\forall i \in [n]: \sum\limits_{j \in [n]} A[i,j] = 1 $}. The language induced by an HMM $A$ is given as: $f_{A}(w) = \pi^{T}.A_{w}.e_{[n]} = (e_{[n]} \otimes \pi)^{T}.vec(A_{w})$, where by definition $A_{w} \myeq \prod\limits_{i=1}^{|w|} A_{w_{i}}$.
For any $n \in \mathbb{N}$, the function induced by an HMM restricted to the support $\Sigma^{n}$ defines a probability distribution. To make the connection with CFGs, we shall design a tensorized representation of HMMs, derived from the observable operator representation. This tensorized representation is given in theorem \ref{tensorizedhmm}, and can be seen as a reshape of the HMM into a CFG given in algebric representation.
\begin{theorem} \label{tensorizedhmm}
 Let $A' = <[n'], \pi', \{A'_{\sigma}\}_{\sigma \in \Sigma}>$ be an HMM. For any string $w \in \Sigma^{*}$, for any factorization $w = w_{1}.w_{2}$, we have  
 $$ f_{A'}(w) = (e_{[n']} \otimes \pi')^{T} \cdot \mathcal{T}_{1} \cdot  (vec(A'_{w_{1}}) \otimes vec(A'_{w_{2}}))$$
 where the tensor $\mathcal{T} \in \mathbb{R}^{n'^{2} \times n'^{2} \times n'^{2}}$ is given as 
 $$\mathcal{T}(\phi(i_{1},j_{1}),\phi(i_{2},j_{2}),\phi(i_{3},j_{3})) \myeq \begin{cases} 1 & \text{if}~~ i_{1} = i_{2} \land j_{1} = j_{3} \land j_{2}=i_{3} \\ 0 & \text{otherwise}\end{cases}$$
 for $\phi:[n'] \times [n'] \rightarrow [n'^{2}]$ an (invertible) mapping such that $\phi(i,j) = (i-1) \cdot n + j$
\end{theorem}
The proof of theorem \ref{tensorizedhmm} is given in appendix A. It relies mostly on careful algebric manipulation of matrices and tensors. Note that the tensor $\mathcal{T}$ is a sparse binary tensor whose number of non-zero elements is equal to $n'^{3}$.
\section{The likelihood problem for unambiguous context-free grammars}
The $UCFG$-Likelihood problem can be formally defined as follows: \\ 
 Let $G = <[n], \pi, \mathcal{M}, \{\beta_{\sigma}\}_{\sigma \in \Sigma}>$ be an algebric description of a UCFG, $A' = <[n'], \pi', \{A'_{\sigma}\}_{\sigma \in \Sigma}>$ be an HMM, and $L > 0$ an integer, the problem is to compute $f_{A'}(L_{G} \cap \Sigma^{L})$. \\
In this section, we shall derive a dynamic algorithm for solving the $UCFG$-Likelihood problem. For this, we first rewrite the quantity of interset $f_{A'}(L_{G} \cap \Sigma^{L})$ in a suitable algebric form. We have, for any $L > 0$
\begin{align}
    f_{A'}(L_{G} \cap \Sigma^{L}) &= \sum\limits_{w \in \Sigma^{L}} \delta_{L_{G}}(w) \cdot f_{A'}(w) \nonumber \\
    &= \sum\limits_{w \in \Sigma^{L}} [\pi^{T} \cdot \beta_{G}(w)]. [(e_{[n']} \otimes \pi')^{T} \cdot vec(A'_{w})] \nonumber \\
    &= \sum\limits_{w \in \Sigma^{L}} (\pi \otimes e_{ [n']} \otimes \pi')^{T} \cdot (\beta_{G}(w) \otimes vec(A'_{w})) \nonumber \\
    &= (\pi \otimes e_{[n']} \otimes \pi')^{T} \cdot \beta_{G,A'}^{(L)} \label{betasum}
\end{align}
where $\beta_{G,A'}^{(L)} \myeq \sum\limits_{w \in \Sigma^{L}} \beta_{G}(w) \otimes vec(A'_{w})$.
The transition from the first to the second equality is due to the fact that, for UCFGs, $f_{G} = \delta_{L_{G}}$, and the third equality is due to the mixed-product property of the kronecker product (i.e. $(AC) \otimes (BD) = (A \otimes B).(C \otimes D)$). 
The expression \eqref{betasum} reduces the computation of the target quantity $f_{A'}(L_{G} \cap \Sigma^{L})$ to the computation of the vector $\beta_{G,A'}^{(L)}$. A recursive formula to obtain this vector is given in the following theorem
\begin{theorem} \label{theoremucfg}
  Let $G = <[n], \pi, \mathcal{M}, \{\beta_{\sigma}\}_{\sigma \in \Sigma}>$ be a UCFG and  $A' = <[n'], \pi', \{A'_{\sigma}\}_{\sigma \in \Sigma}>$ be an HMM. We have \\
  $\bullet ~~\beta_{G,A'}^{(1)} = \sum\limits_{\sigma \in \Sigma} \beta_{\sigma} \otimes vec(A'_{\sigma})$ \\
  $\bullet$ For $L > 0$:
     $~\beta_{G,A'}^{(L)} = (\mathcal{M} \otimes \mathcal{T})_{1} \cdot P \cdot \sum\limits_{i=1}^{L-1}  (\beta_{G,A'}^{(i)} \otimes \beta_{G,A'}^{(L-i)})$ \\
  where $P \in \mathbb{R}^{n'^{4}.n^{2} \times n'^{4}.n^{2}}$ is a permutation matrix that induces the linear map
  $$\forall (i,j,k,l) \in [n]\times[n'^{2}]\times [n]\times [n'^{2}]:~~P \cdot (e_{i}^{(n)} \otimes e_{j}^{(n'^{2})} \otimes e_{k}^{(n)} \otimes e_{l}^{(n'^{2}}) = e_{i}^{(n)} \otimes e_{k}^{(n)} \otimes e_{j}^{(n'^{2})} \otimes e_{l}^{(n'^{2})}$$
\end{theorem}
\begin{proof}
 The first point of the theorem is straightforward from the definition of $\beta_{G,A'}^{(L)}$. Let's prove the second point, we have:
 \begin{align*}
     \beta_{G,A'}^{(L)} &= \sum\limits_{w \in \Sigma^{L}} \beta_{G}(w) \otimes vec(A'_{w}) \\ 
     &= \sum\limits_{w \in \Sigma^{L}} [\sum_{\substack{w_{1},w_{2} \in \Sigma^{+} \\ w_{1}.w_{2} = w}} \mathcal{M}_{1} \cdot (\beta_{G}(w_{1}) \otimes \beta_{G}(w_{2})] \otimes vec(A'_{w}) \\
     &= \sum\limits_{w \in \Sigma^{L}} \sum_{\substack{w_{1},w_{2} \in \Sigma^{+} \\ w_{1}.w_{2} = w}} (\mathcal{M}_{1} \cdot ( \beta_{G}(w_{1}) \otimes \beta_{G}(w_{2})) \otimes (\mathcal{T}_{1} \cdot  vec(A'_{w_{1}}) \otimes vec(A'_{w_{2}})) \\
     &= \sum\limits_{i=1}^{L-1} \sum\limits_{w_{1} \in \Sigma^{i}} \sum\limits_{w_{2} \in \Sigma^{L-i}} (\mathcal{M}_{1} \otimes \mathcal{T}_{1}) \cdot (\beta_{G}(w_{1}) \otimes \beta_{G}(w_{2}) \otimes vec(A'_{w_{1}}) \otimes vec(A'_{w_{2}})) \\
     &= \sum\limits_{i=1}^{L-1} \sum\limits_{w_{1} \in \Sigma^{i}} \sum\limits_{w_{2} \in \Sigma^{L-i}} (\mathcal{M} \otimes \mathcal{T})_{1} \cdot P \cdot (\beta_{G}(w_{1}) \otimes vec(A'_{w_{1}}) \otimes \beta_{G}(w_{2}) \otimes vec(A'_{w_{2}})) \\
     &= (\mathcal{M} \otimes \mathcal{T})_{1} \cdot P \cdot \sum\limits_{i=1}^{L-1} \sum\limits_{w_{1} \in \Sigma^{i}} \sum\limits_{w_{2} \in \Sigma^{L-i}}  (\beta_{G}(w_{1}) \otimes vec(A'_{w_{1}}) \otimes \beta_{G}(w_{2}) \otimes vec(A'_{w_{2}})) \\
     &= (\mathcal{M} \otimes \mathcal{T})_{1} \cdot P \cdot \sum\limits_{i=1}^{L-1} \beta_{G,A'}^{(i)} \otimes \beta_{G,A'}^{(L-i)} 
 \end{align*}
 where the second equality is due to \eqref{betagw}, the third equality is due to lemma \ref{technicallemma} and the associativity property of the kronecker product, and the fourth equality is due to the mixed-product property of the kronecker product. 
\end{proof}
Due to space constraints, we omit the presentation of the detailed pseudo-code for the $UCFG$-Likelihood problem. Basically, the idea will follow the dynamic programming paradigm by constructing iteratively the forward vectors $\{\beta_{G,A'}^{(l)}\}_{l \in [L]}$from $1$ to $L$ according to the recursive formula in theorem \ref{theoremucfg}. A suitable datastructure $H$ (e.g. an array of size $L.n.n'^{2}$) is used to store the intermediate forward vectors. \\
$\bullet$ \textbf{Complexity:} The base case will require $O(|G|\cdot n'^{2} \cdot |\Sigma|)$ operations of addition and multiplication to compute. The most expensive operation in the algorithm is due to the computation of the expression of the form $(\mathcal{M} \otimes \mathcal{T})_{1} \cdot P \cdot (\beta_{G,A'}^{(i)} \otimes \beta_{G,A'}^{(l-i)})$ $l-1$ times to form $\beta_{G,A'}^{(l)}$. A naive way of implementing this operation would be too costly. However, one observes that tensors/matrices involved in this expression enjoy a binary sparse structure that can be leveraged to drastically speed up the computation: $P$ is a permutation matrix which doesn't need to be explicitly constructed, and can be replaced algorithmically by a procedure that maps each dimension to its image. The tensor $\mathcal{T}$ doesn't have to be constructed neither, as its structure as given in theorem \ref{tensorizedhmm} suggests its implementation with 3 nested iterators. The tensor $\mathcal{M}$ is a binary tensor whose number of non-zero elements is smaller than $|G|$. By exploiting these structures, the complexity of computing an expression of the form $(\mathcal{M} \otimes \mathcal{T})_{1} \cdot P \cdot (\beta_{G,A'}^{(i)} \otimes \beta_{G,A'}^{(l-i)})$ would require $O(|G| \cdot n'^{3})$ operations of addition and multiplication. This operation will be performed $l$ times at each iteration $l \in [L]$. Therefore, the overall complexity is $O(|G| \cdot n'^{2} \cdot \max(n'.L^{2}, |\Sigma|))$. The space complexity cost is due to the storage of intermediary results in the datastructure $H$, and it takes $O(n'^{2} \cdot n \cdot L)$.\\
$\bullet$ \textbf{Note.} \textit{
  The closely related problem of sampling UCFGs from HMMs aims at drawing strings $w \in \Sigma^{L}$ according to the probability distribution $f_{G,A'}^{(L)}(w) = \frac{f_{A'}(w)}{f_{A' }(L_{G} \cap \Sigma^{L})} \cdot \delta_{L_{G}}(w)$, when $f_{A'}(L_{G} \cap \Sigma^{L}) \neq 0$. As pointed out by \cite{lafferty91}, the next token probability distribution can be reduced to the computation of prefix probabilities. Same algebric techniques applied in this section can be used to derive a polynomial-time dynamic algorithm to compute prefix probabilities, hence an efficient exact sampler for $f_{G,A'}^{(L)}$. Due to space constraints, we will omit the description of details of the construction of such sampler. We shall assume in the sequel the existence of a polynomial-time running procedure that takes as input a UCFG $G$, an HMM $A'$, and an integer $L$ and draws a sample according to $f_{G,A'}^{(L)}$, and we'll call it $UCFG$-Sample().
}
\section{The likelihood problem of ambiguous CFGs in HMMs}
\subsection{The Hardness of the CFG-Likelihood problem for ambiguity degree less than 2}
In the previous section, the unambiguity property of UCFGs played a crucial role to derive an efficient procedure for the $UCFG$-Likelihood problem. Implementing naively the UCFG-Likelihood algorithm presented in the previous section for arbitrary CFGs will lead to overcounting. On the other hand, it's a well-known fact that for arbitrary CFGs, the likelihood problem is NP-Hard (\cite{kannan95}). 
\begin{algorithm}[]
  \scriptsize
\KwInput{A POLYCFG $G = <[n], \pi, \mathcal{M}, \{\beta_{\sigma}\}_{\sigma \in \Sigma})$ with degree of ambiguity $\Theta(n^{k})$ for a given $k > 0$, an HMM $A' = <[n'], \pi', \{A'_{\sigma}\}_{\sigma \in \Sigma}>$, an integer $L$, $\epsilon \in (0,1)$}
\KwOutput{Approximate of $f_{A'}(L_{G} \cap \Sigma^{L})$}
  $\tilde{Z} \leftarrow \text{UCFG-Likelihood}(G,M,L)$ \\
  \eIf {$\tilde{Z} = 0$}{
     \Return 0
  }{
    $N \leftarrow O((\frac{L^{k}}{\epsilon})^{2})$ \\
    $\tilde{p}_{N} \leftarrow 0 $ \\
    \ForEach{$i$ from $1$ to $O(N)$}{
      $w \leftarrow UCFG-Sample(G,A', L)$ \\
      Compute $f_{G}(w)$ using the recursive formula \eqref{betagw}\\
      $X \sim \text{Bernoulli}(\frac{1}{f_{G}(w)})$  \\
      \If{$X = 1$}{
        $\tilde{p}_{N} \leftarrow \tilde{p}_{N} + 1$
      }
    }
    \Return $\frac{\tilde{p}_{N}.\tilde{Z}}{N}$  }
\caption{POLYCFG-Likelihood}
\end{algorithm}
This suggests that the ambiguity degree of CFGs is intrinsically related to the complexity of the problem. This raises a legitimate question on the complexity class to which belongs the likelihood problem for different subclasses of CFGs according to their ambiguity degree hierarchy. Unfortunately, even with the promise that the ambiguity degree of the CFG is less or equal to 2, the likelihood problem becomes computationnaly hard:
\begin{theorem} \label{hardness}
   Let $G$ be a CFG with an ambigutiy degree less or equal to 2, let $A'$ be an HMM of size $n'$, and $L$ be an integer, there exists no algorithm that runs in polynomial time with respect to the size of $G$, $n'$ and $L'$ and outputs $f_{A'}(L_{G} \cap \Sigma^{L})$ unless $P=NP$.
\end{theorem}
The proof of this theorem can be found in appendix B. It's based on a reduction from the bounded non-emptiness intesection problem between UCFGs, which in turn is reduced to a variant of the bounded Post-Correspondence problem.
\subsection{The POLYCFG-Likelihood problem is FPRAS}
 In this section, we shall give a randomized procedure to approximate the likelihood for the class of POLYCFGs in HMMs with strong probabilistic guarantees. An optimization/counting problem $P$ is said to admit a fully polynomial randomized approximation scheme (FPRAS) if there exists a procedure that takes as input an instance $I$ of the problem, an error parameter $\epsilon \in (0,1)$, runs in time polynmial in the size of the instance and $\frac{1}{\epsilon}$ and outputs a value $OUT$ that satisfies: $ (1 - \epsilon).OPT \leq OUT \leq (1 + \epsilon).OPT$
 with probability greater than $\frac{3}{4}$. $OPT$ refers to the exact output for the instance $I$. 
Algorithm 1 illustrates a randomized algorithm that provides FPRAS guarantees (Theorem \ref{fpras}) for the $POLYCFG$-Likelihood problem. The algorithm uses the $UCFG$ sampler as a proposal distribution in the framework of a accept/reject scheme by introducing a bernoulli random variable, denoted $X$ in the pseudo-code.
The following theorem \ref{fpras} proves the FPRAS guarantees of algorithm 1. The proof can be found in Appendix C.
\begin{theorem} \label{fpras}
  Let $G$ be a POLYCFG with degree of ambiguity $\Theta(n^{k})$ for a constant $k > 0$, and an HMM $A'$. We have
  \begin{enumerate}
      \setlength{\itemsep}{1pt}
  \setlength{\parskip}{0pt}
  \setlength{\parsep}{0pt}
      \item $\mathbb{P}(X = 1 ) \geq \Omega(\frac{1}{L^{k}})$,
      \item For any $\epsilon \in (0,1)$, with probability greater than $\frac{3}{4}$, algorithm 2 outputs a quantity $\tilde{Z}$ that satisfies  $(1 - \epsilon).f_{A'}(L_{G} \cap \Sigma^{L}) \leq \tilde{Z} \leq (1 + \epsilon).f_{A'}(L_{G} \cap \Sigma^{L}) $,
      when the precision parameter given in the input is equal to $\epsilon$
  \end{enumerate}
\end{theorem}

\section*{Conclusion}
We analyzed in this article the problem of computing the likelihood of languages generated by CFGs in HMMs. We showed that the computational complexity of this problem is closely related to the ambiguity of the CFGs. First, we proposed an efficient dynamic algorithm for the likelihood  problem for the case of UCFGs. Then, we showed that the likelihood problem is NP-hard even for the subclass of CFGs with an ambiguity degree smaller or equal to 2. We gave afterwards a randomized algorithm for approximating the likelihood of polynomially ambiguous CFGs with FPRAS guarantees. It's left as an open question if there exists a FPRAS algorithm for the class of exponentially ambiguous CFGs. Another interesting line of research would be to analyze variants of the problem for other classes of sequential probabilistic models such as stochastic context-free grammars and neural-based sequential models such as RNNs and transformers. $\Omega$ 
\bibliography{biblio}

\appendix
\section{}
\textbf{Proof of theorem \ref{tensorizedhmm}} \\
The mode-k vector product of a tensor $\mathcal{T} \in \mathbb{R}^{n_{1} \times ..\times n_{d}}$ with a vector $v \in \mathbb{R}^{n_{k}}$ is a tensor of order $d-1$, denoted $\mathcal{T} \times_{k} v$, such that 
$$ (\mathcal{T} \times_{k} v)[i_{1},..,i_{k-1},i_{k},..,i_{d}] = \sum\limits_{i_{k}=1}^{n_{k}} \mathcal{T}[i_{1},..,i_{d}].v[i_{k}]$$
The mode-k operation is commutative, i.e. for any $(k,k') \in [d]^{2}$ such that $k \neq k'$, we have $\mathcal{T} \times_{k} v \times_{k'} v' = \mathcal{T} \times_{k'} v' \times_{k} v $. For the proof of theorem 1, we will use the following fact: For a tensor $\mathcal{T} \in \mathbb{R}^{n \times n \times n}$ and two vectors $v,~v'$ in $\mathbb{R}^{n}$, we have 
\begin{align} \label{modek}
\mathcal{T} \times_{2} v \times_{3} v' = \mathcal{T}_{1}.(v \bigotimes v')
\end{align}
 Now, we are ready to prove theorem \ref{tensorizedhmm}. We start by proving the following technical lemma:
\begin{lemma} \label{technicallemma}
 Let $A, B$ be two matrices in $\mathbb{R}^{n \times n}$. We have 
 $$ vec(AB) = \mathcal{T} \times_{2} vec(A) \times_{3} vec(B)$$
 where $\mathcal{T} \in \mathbb{R}^{n^{2} \times n^{2} \times n^{2}}$ such that 
 $$\mathcal{T}(\phi(i_{1},j_{1}),\phi(i_{2},j_{2}),\phi(i_{3},j_{3})) \myeq \begin{cases} 1 & \text{if}~~ i_{1} = i_{2} \land j_{1} = j_{3} \land j_{2}=i_{3} \\ 0 & \text{otherwise}\end{cases}$$
 where $\phi:[n] \times [n] \rightarrow [n^{2}]$ is a function such that $\phi(i,j) = (i-1).n + j$
\end{lemma}
\begin{proof}
 Note that the function $\phi$ maps the position of an element in a matrix $A$ in $\mathbb{R}^{n \times n}$ to its corresponding position in its vectorized form $vec(A)$ (i.e. $A[i,j] = vec(A)[\phi(i,j)]$. We have, for any $(i,j) \in [n] \times [n]$,
 \begin{align*}
     (\mathcal{T} \times_{2} vec(A) \times_{3} vec(B))[\phi(i,j)] &= \sum\limits_{(i_{2},j_{2},i_{3},j_{3}) \in [n]^{4}} \mathcal{T}[\phi(i,j), \phi(i_{2},j_{2}), \phi(i_{3},j_{3})]. \\
     &~~~~~~~~~~~~~~~~~~~~~~vec(A)[\phi(i_{2},j_{2})].vec(B)[\phi(i_{3},j_{3})] \\
     &= \sum\limits_{k \in [n]} vec(A)[\phi(i,k)].vec(B)[\phi(k,j)] \\
     &= \sum\limits_{k \in [n]} A[i,k].B[k,j] \\
     &= (AB)[i,j] = vec(AB)[\phi(i,j)]
 \end{align*}
 where the second equality is obtained from the definition of $\mathcal{T}$ in the lemma. 
\end{proof}
Theorem \ref{tensorizedhmm} is a corollary of the previous lemma
\begin{proof}{(Theorem \ref{tensorizedhmm}.)}
  By noting that for any vectors $u,v \in \mathbb{R}^{n}$, for any matrix $A \in \mathbb{R}^{n \times n}$,  $u^{T}.A.v = (v \bigotimes u)^{T}.vec(A)$, we  have 
  \begin{align*}
       f_{A'}(w) &= \pi'^{T}.A'_{w}.e_{[n']} \\
       &= (e_{[n']} \bigotimes \pi')^{T}.vec(A'_{w}) \\
       &= (e_{[n']} \bigotimes \pi')^{T}.vec(A'_{w_{1}}.A'_{w_{2}}) \\
       &= (e_{[n']} \bigotimes \pi')^{T}.(\mathcal{T} \times_{2} vec(A'_{w_{1}}) \times_{3} vec(A'_{w_{2}})) \\
       &=  (e_{[n']} \bigotimes \pi')^{T}.\mathcal{T}_{1}. ( vec(A'_{w_{1}}) \bigotimes vec(A'_{w_{2}}))
  \end{align*}
  where the forth equality is obtained from lemma \ref{technicallemma}, and the last equality is due to \eqref{modek}
\end{proof}
\section{Proof of theorem \ref{hardness}}
We present the $2$-CFG-Likelihood problem given as follows: \\
$\bullet$ \textbf{Instance.} A CFG $G$ with degree of ambiguity less or equal to 2, A HMM $A'$, $p \in (0,1)$, $L$ an integer \\
$\bullet$ \textbf{Problem:} Does $f_{A'}(L_{G} \cap \Sigma^{\leq L}) \geq p$
\begin{theorem} \label{2cfg}
 The $2$-CFG-Likelihood problem is NP-Hard. 
\end{theorem}
The proof proceeds by reduction from the bounded non-emptiness intersection problem of UCFGs. This problem is given as follows: \\
$\bullet$ \textbf{Instance:} Two UCFGs $G_{1},~G_{2}$, an integer $L$ \\
$\bullet$ \textbf{Problem:} Does $L_{G_{1}} \cap L_{G_{2}} \cap \Sigma^{\leq L} = \emptyset$?

We shall show that this problem is NP-Hard. But first, let's show that this problem is reducible in polynomial time to the $2$-CFG-Likelihood problem. 
\begin{lemma} \label{reducible}
 The bounded non-emptiness intersection problem between UCFGs is reducible in polynomial time to the $2$-CFG-Likelihood problem.
\end{lemma}
\begin{proof}
  Let $G_{1},~G_{2}$ be two UCFGs, and $L$ be an integer. Let $A'$ be the trivial single-state HMM that affects equal mass probability to all strings of a given length (i.e. $f_{A'}(w) = \frac{1}{|\Sigma|^{|w|}}$). We have
  \begin{align}  \label{intersection}
       f_{A'}((L_{G_{1}} \cup L_{G_{2}}) \cap \Sigma^{\leq L}) = f_{A'}(L_{G_{1}} \cap \Sigma^{\leq L}) + f_{A'}(L_{G_{2}} \cap \Sigma^{\leq L}) - f_{A'}(L_{G_{1}} \cap L_{G_{2}} \cap \Sigma^{\leq L}) 
  \end{align}
  Since $G_{1}$ and $G_{2}$ are unambiguous, the two terms $f_{A'}(L_{G_{1}} \cap \Sigma^{\leq L})$ and $ f_{A'}(L_{G_{2}} \cap \Sigma^{\leq L})$ can be calculated exactly in polynomial time by calling $L$ times the algorithm UCFG-Likelihood presented in section 2.. Moreover, we can construct in polynomial time a CFG of degree of ambiguity smaller or equal to 2, denoted $G_{1} \cup G_{2}$, that generates the language $L_{G_{1}} \cup L_{G_{2}}$, by adding a new root symbol and two production rules, whose left side is the new root symbol and its right side is the root symbol correponding to $G_{1}$ and $G_{2}$. A normalization into Chomsky form can be done in polynomial time without changing the degree of ambiguity of the original CFG.  Consequently, by \eqref{intersection}, we have $L_{G_{1}} \cap L_{G_{2}} \cap \Sigma^{\leq L} = \emptyset $ if and only if $f_{A'}(L_{G_{1} \cup G_{2}} \cap \Sigma^{\leq L})$ is greater or equal to $f_{A'}(L_{G_{1}} \cap \Sigma^{\leq L}) + f_{A'}(L_{G_{2}} \cap \Sigma^{\leq L})$.
\end{proof}
Next, we shall prove that the bounded non-emptiness intersection problem is NP-Hard, which entails immediately the result of theorem \ref{2cfg} by the reduction given in lemma \ref{reducible}. 
\begin{lemma} \label{bounded}
  The bounded non-emptiness problem between UCFGs is NP-Hard.
\end{lemma}
We'll reduce this problem from a variant of the bounded Post Correspondence problem (BPCP) given as follows: Given a pair of $N$ strings  $\{(u_{i},v_{i})\}_{i \in [N]}$, and integer $M$:  Does there exist a sequence $i_{1},..i_{k}$ of elements in $[N]$ such that  $u_{i_{1}}...u_{i_{k}} = v_{i_{1}}..v_{i_{k}}$ and  $k + |u_{i_{1}}..u_{i_{k}}| \leq M$? \\
The proof of NP-Hardnss of the BPCP will be given later. For now, assume it's the case. The details of the reduction are similar to the classical proof of the undecidability of the non-emptiness intersection problem of two UCFGs \cite{}. In the following, we will give the important outlines of this construction, and we refer the interested reader to \cite{asveld00}, for more details. Let $\{(u_{i},v_{i})\}_{i \in [N]}$ be a pair of strings. Define the extended alphabet $\Sigma_{e} = \Sigma \cup \{1,2..,N\} \cup \{\#,\&\}$. The idea is to construct (in polynomial time) two UCFGs $G_{1}$ and $G_{2}$, where the $G_{1}$ generates the language $L_{G_{1}} = i_{1}..i_{k}\&j_{l}..j_{1}\#v_{j_{1}}..v_{j_{l}}\& u_{i_{k}}^{R}..u_{i_{1}}^{R}$, and $G_{2}$ generating the language $L_{G_{2}} = w\&w^{R}\#w'\&w'^{R}$, where $.^{R}$ is the mirror operation. The language $L_{G_{1}} \cap L_{G_{2}}$ corresponds exactly to words of the form $i_{1}..i_{k}\&j_{k}..j_{1}\#v_{j_{1}}..v_{j_{l}}\& u_{i_{k}}^{R}..u_{i_{1}}^{R}$ where $\forall l \in [k]: i_{l} = j_{l}$, and $u_{i_{1}}...u_{i_{k}} = v_{i_{1}}..v_{i_{k}}$. The length of such string is equal to $2(k + |u_{i_{1}}..u_{i_{k}}|) + 3$. Consequently, for any $M > 0$, there exists a word in $L_{G_{1}} \cap L_{G_{2}}$ of length smaller than $ 2M + 3$ if and only if there exists a sequence that satisfies the conditions of the BPCP problem.\\
To complete the proof of \ref{bounded}, we need to prove the NP-Hardness of the variant of the bounded post-correspondence problem defined above.
\begin{lemma}
   The BPCP is NP-Hard. 
\end{lemma}
\begin{proof}
   We reduce any language in NP to our variant of the Post-Correspondence problem. Let $L$ be any language in NP. Thus, there exists a non-deterministic Turing Machine (NTM) $N_{A}$, and a polynomial $p_{A}$ such that for any string $w$, the machine $N_{A}$ decides if $w \in L$ in a time at most $p_{A}(|w|)$. The construction (in polynomial time with respect to the size of the NTM and the length of the string) of an instance of the PCP problem, denoted $I_{N_{A},w},$ that simulates a turing machine on a string $w$ is classical for the proof of the PCP undecidability result. The PCP instance generates two sequences of computation configurations of a turing machine such that it accepts if and only if there exists a correspondence between the two sequences.  We refer the reader to \cite{sipser13} for details of the construction. What's left is to find an upper bound of the sum of the list of pair of strings used with the length of such sequence. In what follows, we'll settle for the use big $O$ notation to dervie such upper bound. First, by definition of $p_{A}$, $k \leq O(p_{A}(|w|))$. Second, observe that a configuration of $N_{A}$ (i.e. the content of the turing machine tape) while running on  $w$ as input can't exceed $O(p_{A}(|w|))$. Hence, the length of the sequence of configurations of the NTM $N_{A}$ during a run on $w$ can't exceed $O(p_{A}^{2}(|w|))$. So, $N_{A}$ accepts $w$ if and only if there exists a solution of the PCP problem in $I_{N_{A},w}$ whose sum of the pair of strings used and the length of the string is smaller than $O(p_{A}^{2}(|w|))$. 
\end{proof}
Now, we are ready to prove theorem \ref{hardness}. \\ \textit{(Proof theorem \ref{hardness}.)} Suppose there exists an algorithm $\mathcal{A}$ that can compute $f_{A'}(L_{G} \cap \Sigma^{L})$ in polynomial time. Then, by calling $\mathcal{A}$ $L$ times, we can compute in polynomial time $f_{A'}(L_{A} \cap \Sigma^{\leq L})$. Hence, we can decide the $2$-CFG-Problem in polynomial time. Since this problem is NP-Hard by theorem \ref{2cfg}, this is possible only if $P = NP$, which completes the proof.
\section{}
\begin{proof}{(Theorem \ref{fpras})}
 We first write the term $\mathbb{P}(X = 1)$ in a suitable manner,
 We have 
 \begin{align}
\mathbb{P}(X = 1) &= \sum\limits_{w \in \Sigma^{L}} \mathbb{P}(w , X = 1) \nonumber \\
& = \sum\limits_{w \in \Sigma^{L}} \frac{f_{G}(w).f_{A'}(w)}{\sum\limits_{w \in \Sigma^{L}} f_{G}(w).f_{A'}(w)}.\frac{1}{f_{G}(w)} \nonumber \\
&= \sum\limits_{w \in \Sigma^{L}} \frac{f_{A'}(w). \delta_{G}(w)}{\sum\limits_{w \in \Sigma^{L}} f_{G}(w).f_{A'}(w)} \nonumber \\
&= \frac{f_{A'}(L_{G} \cap \Sigma^{L})}{\sum\limits_{w \in \Sigma^{L}} f_{G}(w).f_{A'}(w)} \label{bern}
\end{align}
Now we can prove the statement 1. \\
(1) Let $G$ be a POLYCFG with degree of ambiguity $\Theta(n^{k})$ for some $k > 0$. We have, then
$$ \forall w \in \Sigma^{L}:~f_{G}(w) \leq \delta_{G}(w).O(L^{k}) $$ 
Multiplying each term by $f_{A'}(w)$ and summing up, we obtain 
\begin{align*}
     \sum\limits_{w \in \Sigma^{L}} f_{G}(w).f_{A'}(w) &\leq O(L^{k}). \sum\limits_{w \in \Sigma^{L}} \delta_{G}(w).f_{A'}(w) \leq O(L^{k}) f_{A'}(L_{G} \cap \Sigma^{L}) \\
     \implies   \mathbb{P}(X = 1) & = \frac{f_{A'}(L_{G} \cap \Sigma^{L})}{\sum\limits_{w \in \Sigma^{L}} f_{G}(w).f_{A'}(w)}\geq \Omega(\frac{1}{L^{k}})
 \end{align*}
(2) If $f_{A'}(L_{G} \cap \Sigma^{L}) = 0$, then algorithm 2 outputs 0. Indeed, we have $\sum\limits_{w \in \Sigma^{L}} \delta_{L_{G}}(w).f_{A'}(w) = 0$ if and only if $\sum\limits_{w \in \Sigma^{L}} f_{G}(w).f_{A'}(w) = 0$. Next, let's focus on the case $f_{A'}(L_{G} \cap \Sigma^{L}) \neq 0$. Let $\tilde{p}_{N}$ denote the empirical estimate of $\mathbb{P}(X = Accept)$ with a sample size equal to $N$. By Hoeffding's inequality, for $N \geq O((\frac{L^{k}}{\epsilon})^{2})$, we have with probability greater than $\frac{3}{4}$
$$ |\mathbb{P}(X = 1) - \tilde{p}_{N} | \leq \frac{\epsilon}{L^{k}} \implies | \frac{\tilde{p}_{N}}{\mathbb{P}(X=1)} - 1 | \leq \frac{\epsilon}{\mathbb{P}(X = 1).L^{k}} \leq \epsilon$$
where the second inequality is due to the first statement of the theorem.\\
Consequently, replacing $\mathbb{P}(X = 1)$ in the inequality above by its expression in \eqref{bern}, we obtain, with probability greater than $\frac{3}{4}$,
$$ | \frac{\tilde{p}_{N}.\sum\limits_{w \in \Sigma^{L}} f_{G}(w).f_{A'}(w)}{f_{A'}(L_{G} \cap \Sigma^{L})} - 1 | \leq \epsilon$$
\end{proof}

\end{document}